\documentclass[10pt, conference, letterpaper]{IEEEtran}

\usepackage[
backgroundcolor = White,textwidth=\marginparwidth]{todonotes}

\usepackage[markup=underlined]{changes}

\usepackage[utf8]{inputenc} 
\usepackage[T1]{fontenc}    
\usepackage{url}            
\usepackage{booktabs}       
\usepackage{amsfonts}       
\usepackage{nicefrac}       
\usepackage{microtype}      

\usepackage{graphicx}
\usepackage{color}
\usepackage{rotating}
\usepackage{tabularx}
\usepackage{pdflscape}
\usepackage{amsmath,amsthm,amssymb}
\usepackage{algorithm,algorithmic}
\usepackage{bbm,dsfont}
\usepackage{caption,subcaption}
\usepackage{wrapfig}
\usepackage{cancel}
\usepackage{enumerate, cases}
\usepackage{thmtools,thm-restate}
\usepackage{mathtools}

\usepackage[none]{hyphenat}
\usepackage{multirow}
\usepackage{makecell}
\usepackage{setspace}

\allowdisplaybreaks

\usepackage{dblfloatfix}
\usepackage{array}
\usepackage{enumitem}

\allowdisplaybreaks

\DeclareMathOperator*{\argmax}{arg\,max}

\newcolumntype{C}[1]{>{\centering}m{#1}}
\newcommand{\EE}[1]{\mathbb{E}\left[#1\right]}

\newcommand{\Prob}[1]{\mathbb{P}\left\{#1\right\}}
\newcommand{\ceil}[1]{\left\lceil #1 \right\rceil}
\newcommand{\floor}[1]{\left\lfloor #1 \right\rfloor}
\newcommand{\Regret}{\mathcal{R}}
\newcommand{\A}{\mathcal{A}_{C}}
\newcommand{\R}{\mathbb{R}}

\newcommand{\one}[1]{\mathds{1}_{\left\{#1\right\}}}

\newcommand{\ONUM}{\mathcal{P}_{\small{\text{ONUM}}}}

\newcommand{\bmu}{\boldsymbol{\mu}}
\newcommand{\btheta}{\boldsymbol{\theta}}

\newcommand{\bx}{\boldsymbol{x}}

\theoremstyle{plain}
\newtheorem{thm}{Theorem}

\newtheorem{cor}{Corollary}

\newtheorem{defi}{Definition}

\begin{document}

	\title{
		Stochastic Network Utility Maximization with Unknown Utilities: Multi-Armed Bandits Approach
		}
	\author{
		Arun Verma and Manjesh K. Hanawal \\
		Industrial Engineering and Operations Research\\
		Indian Institute of Technology Bombay, India \\ 
		\texttt{\{v.arun, mhanawal\}@iitb.ac.in}
	}
	\maketitle
	
	\begin{abstract}
        In this paper, we study a novel Stochastic {\em Network Utility Maximization} (NUM) problem where the utilities of agents are unknown. The utility of each agent depends on the amount of resource it receives from a network operator/controller. The operator desires to do a   resource allocation that maximizes the expected total utility of the network. We consider threshold type utility functions where each agent gets non-zero utility if the amount of resource it receives is higher than a certain threshold. Otherwise, its utility is zero (hard real-time).  We pose this NUM setup with unknown utilities as a regret minimization problem. Our goal is to identify a policy that performs as `good' as an oracle policy that knows the utilities of agents. We model this problem setting as a bandit setting where feedback obtained in each round depends on the resource allocated to the agents. We propose algorithms for this novel setting using ideas from Multiple-Play Multi-Armed Bandits and Combinatorial Semi-Bandits. We show that the proposed algorithm is optimal when all agents have the same utility. We validate the performance guarantees of our proposed algorithms through numerical experiments.
	\end{abstract}
	

	\IEEEpeerreviewmaketitle
	\bstctlcite{IEEEexample:BSTcontrol} 	

	\section{Introduction}
	\label{sec:introduction}

Network Utility Maximization (NUM) is an approach for resource allocation among multiple agents such that the total utility of all the agents (network utility) is maximized. In its simplest  form, NUM solves the following problem:
\begin{align*}
	\underset{\bx}{\text{maximize}} & \sum_{i=1}^{K}U_i(x_i)\\
	& \mbox{subject to}  &\hspace{-3cm}\sum_{i=1}^K x_i \leq C
\end{align*}
where $U_i(\cdot)$ denotes the utility of agent $i$, variable $\bx=(x_1,x_2,\ldots, x_K) \in \R_+^K$ denote the allocated resource vector, and $C \in \R_+$ is  amount of resource available. Utilities define the satisfaction level of the agents, which depend on the amount of resource they are allocated. A resource could be bandwidth, power, or rates they receive. Since the seminal work of Kelly \cite{ETT1997_ChargingAndRateControl}, there has been a tremendous amount of work on NUM and its extensions. NUM is used to model various resource allocation problems and improve network protocols based on its analysis. The nature of utility functions is vital in the analysis of the NUM problem and assumed to be known or can be constructed based on the agent behavior model and operator cost model. However, agent behavior models are often difficult to quantify. In this work, we study the NUM problem where the utilities of the agent are unknown and stochastic.

The earlier NUM problems considered deterministic settings. Significant progress has been made to extend the NUM setup to take into account the stochastic nature of the network and agent behavior \cite{ETT2008_StochasticNUM_YiChiang}. For both the static and stochastic networks, the works in the literature often assume that the utility functions are smooth concave functions and apply Karush-Kuhn-Tucker conditions to find the optimal allocation. However, if the utility functions are unknown, these methods are useful only once the utilities are learned. Many of the NUM variants with full knowledge of utilities aim to find an optimal policy that meets several constraints like stability, fairness, and resource \cite{INFOCOM2010_DelayBasedNUM_Neely,WiOpt2017_DRUM_EryilmazKoprulu,WiOpt2018_NUMHetrogeneous_SinhaModiano}. In this work, we only focus on resource constraint due to limited divisible resource (bandwidth, power, rate).

Since learning an arbitrary utility function is not always feasible,  we assume the utilities belong to a class of `threshold' type functions. Specifically, we assume that the utility of each agent is stochastic with some positive mean only when it is allocated a certain minimum resource.  We refer to the minimum resource required by an agent as its `threshold' and the mean utility it receives when it is allocated resource above the threshold as its `mean reward.' Thus the expected utility of each agent is defined by two parameters -- a threshold and a mean reward.  Such threshold type utilities correspond to hard resource requirements. For example, an agent can transmit and obtain a positive rate  (reward) only if its power or bandwidth allocation is above a certain amount.

In each round, the operator allocates a resource to each agent and observes the utilities the agent obtains. The goal of the operator is to allocate resource such that the expected network utility is maximized. We pose this problem as a Multi-Armed Bandit (MAB) problem where the operator corresponds to a learner, agents to arms, and utilities to rewards. The learner's goal is to learn a policy that minimizes the difference between the best achievable expected network utility with full knowledge of the agent utilities and that obtained by the learner under the same resource constraint with the estimated utilities of agents.

The reward structure in the MAB formulation of the NUM problem is different from the standard MAB problem. Hence one cannot directly apply the standard MAB algorithms to the NUM setting.  Unlike standard MAB setup where the reward depends on the arm played, in the NUM setup, the reward obtained in each round depends on the resource allocated by the learner. The learner observes the utility of an agent only when it allocates resource above its threshold. Otherwise, it gets no reward on the utilities of the agents. Further, in the NUM setup, the learner may observe utilities of more than one agent in each round depending on how many agents receive resource above their corresponding thresholds.

A good policy for the NUM setting needs to learn the expected utility for each agent, i.e., the thresholds, and mean rewards. We first consider the case where the threshold for each agent is the same and then consider the case where the thresholds could be different. For both cases, we develop a policy based on Thompson Sampling that achieves sub-linear regret. Our contributions can be summarized as follows:
\begin{itemize}
	\item In Section \ref{sec:problemSetting}, we give a novel model for Online Network Utility Maximization (ONUM) with unknown utilities.
	
	\item In Section \ref{sec:same_theta}, we study the symmetric case where the threshold is the same for all the agents. Using the concept of `allocation equivalent,' we develop an optimal algorithm named ONUM-ST by exploiting connection with Multiple-Play Multi-Armed Bandits to our setting.
	
	\item In Section \ref{ssec:different_theta}, we study a more general asymmetric case where the threshold for agents could be different. We develop an efficient algorithm named ONUM-DT by exploiting its connection to Combinatorial Semi-Bandits.
	
	\item We empirically validate the performance of our algorithms via experiments on synthetic problems in Section \ref{sec:experiments}.
\end{itemize}

\subsection{Related Work}
NUM has been an active area of research in the past two decades. Many of its variants are developed for resource allocation in networking. We refer the readers to \cite{JSAC2006_TutorialOnNUM_PalomarChinag},\cite{TAC2007_AlternateDistributed_PalomarChinag} for an informative tutorial and survey on this subject. In this Section, we discuss works that look into learning aspects in NUM.

NUM in a multi-agent network is studied with partially observable channel states in \cite{PEVA2013_NetworkUM_LiNeely}. The authors assume that the channel states are Markovian and exploit the memory in the channel to maximize a known concave function of time average reward using the framework of Restless Bandits.  Stochastic Multi-Armed Bandits (MAB) \cite{ML2002_FiniteTimeAnlaysis_Auer,Book2012_RegretAnlaysis_Bubeck} are  applied in distributed optimization in networks. In cognitive radio networks (CRNs) with multiple agents, the MAB setup is used to maximize network throughput in a distributed setting \cite{JSAC2011_DistributedAlgorithms_Anandakumar,TCNS2016_DistributedAlgorithms_Anandakumar,ALT2018_MultiplayerBandits_BessonKaufmann,INFOCOM2019_DistributedLearning_TibrewalPatchalaHanawal,WiOpt2019_DistributedAlgorithms_VermaHanawalVaze}. The fairness issues while maximizing the network utility using the MAB setting is studied in \cite{INFOCOM2019_CombinatorialSleeping_LiLiuJi}.

Our MAB formulation of NUM involves solving a combinatorial 0-1 knapsack problem. Bandits with Knapsacks studied in \cite{JACM18_badanidiyuru2018bandits} also require solving a knapsack problem in each round. However, in their model, resource gets consumed in every round, unlike ours. Also, in Bandits with Knapsacks, the resource allocation does not affect the reward observed. \cite{NIPS16_abernethy2016threshold,ICML18_jain2018firing} also assume some threshold model for rewards. However, in their model, an agent receives a reward only if the sampled reward from its associated distribution is above some threshold. Whereas in our setup, the threshold corresponds to the minimum resource required. Resource allocation with semi-bandits feedback   \cite{UAI14_lattimore2014optimal,NIPS15_lattimore2015linear,ALT18_dagan18a} study a related but less general setting where the reward is observed in each round irrespective of the amount of resource allocated. Whereas in our setting, it is not the case as the reward is zero if the minimum requirement of the resource is not satisfied.
The adaptive resource allocation problem is also studied in loss setting with censored feedback by \cite{NeurIPS19_verma2019censored}, where no loss values are observed from arms that receive more resource than their associated thresholds. In this work, we consider a reward setting, and our algorithms differ from that in \cite{NeurIPS19_verma2019censored} as \cite{NeurIPS19_verma2019censored} first estimate the threshold value associated with the arms and then estimate the mean losses of arms. Whereas our goal is to maximize total reward, and our algorithms jointly estimate both threshold and mean reward of the arm. 

Depending on the resource allocated in each round, we observe the reward from a subset of agents who get their minimum required resource. Such combinatorial aspects of arms play are widely studied as Combinatorial (Semi-)Bandits in \cite{NIPS15_combes2015combinatorial,ICML15_komiyama2015optimal,NIPS16_chen2016combinatorial, ICML18_wang2018thompson}. Though these works are not directly related to our setup,  we make explicit connections of our algorithms to the algorithms given in \cite{ICML15_komiyama2015optimal} and \cite{ICML18_wang2018thompson}.


	\section{Problem Setting}		
	\label{sec:problemSetting}

We consider an online version of the NUM problem where utilities of the agents are unknown, and the network operator aims to reach the optimal resource allocation via sequential allocations. Let $K$ denote the number of agents, and $C$ denotes the amount of divisible resource (bandwidth, power). The operator assigns a fraction of resource to each agent, and the utility of agents depends on the amount of resource they receive. Utility for agent $i \in [K]$ where $[K] := \{1, 2, \ldots, K\}$, is stochastic and drawn from a fixed distribution $\nu_i$ with support in $[0,1]$ and mean $\mu_i \in [0,1]$ in each round, provided it receives a certain minimum amount of resource, otherwise its utility is zero. For each $i \in [K]$, let $\theta_i \in [0,C]$ denote the minimum resource required for agent $i$ to obtain non-zero utility. Then, for each agent $i \in  [K]$ utility is parameterized as $(\theta_i, \mu_i)$ such that agent $i$ receives mean utility $\mu_i$ if it is allocated at least $\theta_i$ fraction of resource, otherwise its utility is zero. 

The resource allocated to the agents decides the reward observed by the operator. If the allocated resource is at least $\theta_i$ for agent $ i \in [K]$, the operator observes the realization of the utility obtained by the agent drawn from the distribution $\nu_i$. Otherwise, zero utility is obtained by the agent.

In the following, we assume that each $\nu_i, i\in [K]$ is a Bernoulli distribution with parameter $\mu_i$. It is a challenging setting as the operator can't know whether sufficient resource is allocated to an agent whenever the agent receives zero utility. Because with Bernoulli utility the agent $i \in [K]$ can receive zero utility even if it is allocated minimum required resource with probability $(1-\mu_i).$ Whereas this probability is very small (almost zero) if the utility distribution is continuous.

Following the terminology of Multi-Armed Bandits (MAB), henceforth we refer to agents as arms, operator as learner and utility as a reward. Let $\bx:=\{x_i: i\in [K]\}$, where $x_i \in [0,C]$, denotes the resource allocated to arm $i$. An allocation vector $\bx$ is said to be feasible if $\sum_{i=1}^K x_i \leq C$. The set of all feasible allocations is denoted as $\A$. For any $\bx \in \A$,  mean reward from arm $i$ is non-zero only if $x_i\geq \theta_i$.  The goal of the learner is to find a feasible resource allocation that maximizes the network utility. 

The available resource may be allocated to multiple arms in our setup. However, the reward from each arm may not be observed depending on the amount of resource allocated to them. Hence we have semi-bandit feedback in each round, and we refer to this setup as Online Network Utility Maximization (ONUM). The vectors $\bmu:=\{\mu_j\}_{i \in [K]}$ and $\btheta:=\{\theta_i\}_{ i\in [K]}$ are unknown and identify an instance of ONUM problem. Henceforth we identify an ONUM instance as $P=(\bmu,\btheta,C) \in [0,1]^{K}\times \R_+^K \times \R_+$ and denote collection of ONUM instances as $\ONUM$. For simplicity of discussion, we assume that arms are indexed according to their decreasing mean rewards, i.e., $\mu_1 \geq \mu_2, \ldots, \geq \mu_K$, but the algorithms are not aware of this ordering. We refer to the first $M$ arms as \emph{top-}$M$ arms. For instance $P \in \ONUM$, the optimal allocation can be computed as the following $0$-$1$ knapsack problem:
\vspace{-1mm}
\begin{equation*}
	\label{equ:networkUtility}
	\bx^\star  \in \argmax _{\bx \in \A} \sum_{i=1}^K\mu_i \one{x_i\ge \theta_i}.
\end{equation*}

The interaction between a learner and the environment that governs rewards for the arms is as follows: In the round $t$, the environment generates a reward vector $(Y_{t,1}, Y_{t,2},\ldots, Y_{t,K}) \in \{0,1\}^K$, where $Y_{t,i}$ denotes the true reward for arm $i$ in round $t$.  The sequence $(Y_{t,i})_{t\geq 1}$ is generated i.i.d. with the common mean $\EE{Y_{t,i}}=\mu_i$ for each $i \in [K]$. The learner selects a feasible allocation $\bx_t=\{x_{t,i}: i\in [K]\}$ and observes reward vector $Y_t^\prime=\{Y^\prime_{t,i}: i\in [K]\}$, where $Y_{t,i}^\prime=Y_{t,i}\one{x_{t,i}\ge \theta_i}$ and collects reward $r_t(\bx_t)=\sum_{i \in [K]}Y_{t,i}^\prime$. A policy of the learner is to select a feasible allocation in each round based on the observed reward such that the cumulative reward is maximized. The performance of a policy that makes allocation $\{\bx_t\}_{t\geq1}$ in round $t$ is measured in terms of expected (pseudo) cumulative regret for $T$ rounds given by
\vspace{-1mm}
\begin{equation*}
	\mathbb{E}[\Regret_T] = T\sum_{i=1}^K\mu_i \one{x^\star_i\ge \theta_i} - \EE{ \sum_{t=1}^T\sum_{i=1}^K Y_{t,i} \one{x_{t,i}\ge \theta_i}}.
\end{equation*}
A good policy must have sub-linear cumulative regret, i.e., $\EE{\Regret_T}/T \rightarrow 0$ as $T \rightarrow \infty$.

\subsection{Allocation Equivalent}
Next, we define the notion of treating a pair of thresholds for the given loss vector and resource to be `equivalent.'
\begin{defi}[Allocation Equivalent]
	For any reward vector $\bmu$ and fix amount of resource $C$, two threshold vectors $\btheta$ and $\hat{\btheta}$ are {\em allocation equivalent} iff the following holds:
	\begin{equation*}
		\max_{\bx \in \A} \sum_{i=1}^K\mu_i \one{x_i\ge \theta_i}  = \max_{\bx \in \A} \sum_{i=1}^K\mu_i \one{x_i\ge \hat{\theta}_i}.
	\end{equation*}
\end{defi}

Such equivalence allows us to find the threshold vector within fix error tolerance, which has the same total mean reward reduction as a true threshold vector has.

	\section{Same Threshold for All Arms}
	\label{sec:same_theta}

We first focus on the special case of the online network utility maximization problem where $\theta_i=\theta_s$ for all $i \in [K]$. With abuse of notation, we denote an instance of ONUM with the same threshold as $(\bmu, \theta_s, C)$ where $\theta_s \in [0, C]$ is the value of the same threshold. Note that even though the threshold is the same, the mean rewards can be different across the arms. 
\noindent
Though $\theta_s$ can be any value in the interval $[0, C]$, an allocation equivalent to it can be restricted to a finite set. Our next result shows that the search for an allocation equivalent can be confined to a set consisting of $K$ elements.
\begin{restatable}{lem}{ThetaSet}
	\label{lem:thetaSet}
	Let $\theta_s \in [0,C]$, $M=\min\{\lfloor C/\theta_s \rfloor,$ $K\}$ and $\hat{\theta}_s=C/M$. Then  $\theta_s$ and $\hat{\theta}_s$ are allocation equivalent. Further, $\hat{\theta}_s \in \Theta$ where $\Theta = \{ C/K, C/(K-1), \cdots, C\}$. 
\end{restatable}
\begin{proof}
	The proof is a straight forward adaption of Lemma 1 in \cite{NeurIPS19_verma2019censored} by allowing the threshold to be any value in $[0,C]$ where $C$ can be greater than $1$. The case when $\floor{C/\theta_s} \geq K$ is trivial. Let consider the case when $\floor{C/\theta_s} < K$. Using the definition of $M$, we have $M \le C/\theta_s$ and $\theta_s \le C/M \doteq \hat{\theta}_s$. Hence $\hat{\theta}_s \ge \theta_s$. Therefore allocation of $\hat{\theta}_s$ or $\theta_s$ fraction of resource allocation to an arm achieves the same mean reward. Further, for both instances $(\bmu, \theta_s, C)$ and $(\bmu, \hat{\theta}_s, C)$, the optimal allocations collect reward from the top-$M$ arms and no reward from the remaining arms. Hence the mean reward collected from the optimal allocations in both the instances results in the same total mean reward. It completes the proof of first part of lemma. Since all arms have same threshold, learner has to equally distribute resource among selected top-$M$ arms. As $M \in \{1, \ldots, K\}$ and $\hat{\theta}_s \le C$, the desired value of $\hat{\theta}_s$ is the one of element in set $\Theta = \{ C/K, C/(K-1), \cdots, C \}$. 
\end{proof}

Once the threshold is known, the optimal allocation of a learner is to allocate $\hat{\theta}_s$ amount of resource to each of the top-$M$ arms where $M=C/\hat{\theta}_s$. Lemma \ref{lem:thetaSet} shows that an allocation equivalent $\hat{\theta}_s$ for any instance $(\bmu,\theta_s,C)$ is one of value in a finite set $\Theta$. Once allocation equivalent is known, the problem reduces to identifying the top-$M$ arms and then allocating $\hat{\theta}_s$ amount of resource to each one of them to maximize the total mean reward. The latter part is equivalent to solving Multiple-Play Multi-Armed Bandits, as discussed next.

\subsection{Multiple-Play Multi-Armed Bandits (MP-MAB)}
In stochastic Multiple-Play Multi-Armed Bandits, a learner can play a subset of arms in each round. The selected subset of arms is of fixed size (known) and also known as superarm \cite{TAC1987_MultiPlayBandits_Anatharam}. The mean reward of a superarm is the sum of the mean reward of its constituent arms. In every round, a learner selects a superarm and observes the reward from each selected arm (semi-bandit feedback). The goal of the learner is to select a superarm that has the maximum mean reward. In MP-MAB, a policy selects a superarm in each round based on the previous reward information. The performance of any policy is measured in terms of regret. The regret is the difference between cumulative reward collected by playing optimal superarm and that collected by the policy in each round.

\noindent
{\bf Lower bound}: Due to the equivalence between the MP-MAB and ONUM problem with the (known) same threshold, the lower bound for MP-MAB is also a lower bound for the ONUM problem with the same threshold. Therefore, the following lower bound given for a strongly consistent algorithm by Theorem 3.1 in \cite{TAC1987_MultiPlayBandits_Anatharam} is also a lower bound on the ONUM problem with known same threshold:
\begin{equation}
	\label{eqn:LowerBound}
	\lim_{T\rightarrow \infty}\frac{\mathbb{E}[\Regret_T]}{\log T} \geq \sum_{i \in [K]\setminus[M]} \frac{\mu_M - \mu_i}{d(\mu_i, \mu_M)}
\end{equation}
where $d(p,q)$ is the Kullback-Leibler ($KL$) divergence between two Bernoulli distributions with parameter $p$ and $q$.

Once the threshold is known, any algorithm that works well for the MP-MAB also works well for the ONUM. Hence one can apply algorithms like ESCB \cite{NIPS15_combes2015combinatorial} and MP-TS \cite{ICML15_komiyama2015optimal} once an allocation equivalent is found for $\theta_s$. MP-TS uses Thompson Sampling, whereas ESCB uses UCB  and kl-UCB type indices.
We can adapt any of these algorithms for our setting. But we use MP-TS to our as it gives better empirical performance compare to ESCB and has been shown to obtain optimal regret bound for Bernoulli reward distributions.

\subsection{Algorithm ONUM-ST} 
We develop an algorithm named Online Network Utility Maximization with the Same Threshold (ONUM-ST). It exploits the result of Lemma \ref{lem:thetaSet} to learn an allocation equivalent of threshold and adapts MP-TS to minimize the regret. The pseudo-code of ONUM-ST is given in Algorithm \ref{alg:ONUM-ST}. ONUM-ST works as follows: it takes $K,C,\delta$ and $\epsilon$ as input where $\delta$ is the confidence on the correctness of estimated allocation equivalent and $\epsilon$ is such that $\mu_K \ge \epsilon > 0$. We set $\Theta=\{C/K, C/(K-1), \dots, C\}$ as in Lemma \ref{lem:thetaSet}. The elements of $\Theta$ are in increasing order, and each element is a candidate for allocation equivalent of  $\theta_s$ (line $2$). We also set the prior distribution for the mean reward of each arm as the Beta distribution $\beta(1, 1)$. For each arm $i \in [K], S_i$ represents the number of rounds when the reward is $1$, and $F_i$ represents the number of rounds when the reward is $0$ whenever the arm $i$ receives resource above its threshold.

ONUM-ST finds a threshold $\hat{\theta}_s$ that is an allocation equivalent to $\theta_s$ with high probability (at least $1-\delta$) using binary search over the set $\Theta$. The search begins by taking $\hat{\theta}_s$ to be the middle element in $\Theta$ (line 5). Let $S_i(t)$ and $F_i(t)$ denote the values of $S_i$ and $F_i$ in the starting of the round $t$. In round $t$, a sample $\hat\mu_i$ is drawn from $\beta(S_i(t), F_i(t))$ for each arm $i \in [K]$ independent of other arms (line $5$).  $\hat\mu_i$ values are ranked as per their decreasing values and each of the top-($C/\hat\theta_s$)  (denoted as set $A_t$ in line $6$) arms are allocated $\hat\theta_s$ amount of resource and their rewards are observed (line $7$). After knowing allocation equivalent of threshold, only $S_i$ and $F_i$ are updated for each arm $i \in A_t$ (line $17$).

Before knowing allocation equivalent (line $8$), if a reward $1$ is observed at any of the arms in the set $A_t$ (line $9$), it implies that the current value of $\hat{\theta}_s$ is possibly an overestimate of $\hat{\theta}_s$. So all candidates larger than $\hat{\theta}_s$ in set $\Theta$ are removed, and the search is repeated in the remaining half of the elements by starting with the middle element (line $10$). The success and failure counts are also updated as $S_i = S_i + Y_{t,i}, F_i = F_i + 1-Y_{t,i}+Z_i$ for each arm $i \in A_t$, and for all $k \in K \setminus A_t$ only failure count is updated as $F_k=F_k+Z_k$ (line $11$). The variable $Z_i, ~\forall i\in [K]$ keeps track of how many times $0$ is observed for arm $i$ before $1$ is observed on it when it is allocated resource. It is reset to zero once a reward $1$ is observed for any arm in set $A_t$. Variable $Z_i, i \in [K]$ allow us to distinguish the zeros observed when the arm receives over and under allocation of resource.
\vspace{-1mm}
\begin{algorithm}[H] 
	\caption{\bf ONUM-ST}
	\label{alg:ONUM-ST}
	\begin{algorithmic}[1]
		\STATE \textbf{Input:} $K, C, \delta, \epsilon$
		\STATE Initialize $\Theta$ as in Lemma \ref{lem:thetaSet}, $W_c=0, l=0, u = K, j = \ceil{u/2}, \forall i \in [K]: S_i=1, F_i=1, Z_i=0$ 
		\STATE $W_\delta = {\log(\log_2(K)/\delta)}/({\log(1/(1-\epsilon))})$ 
		\FOR{$t=1,2,\ldots,$}
			\STATE Set $\hat{\theta}_s = \Theta[j]$ and $\forall i \in [K]: \hat{\mu}_i \leftarrow \beta(S_i, F_i)$
			\STATE $A_t \leftarrow$ set of top-$({C}/{\hat{\theta}_s})$ arms from estimates $(\hat{\mu}_i)$
			\STATE $\forall i \in A_t:$ allocate $\hat{\theta}_s$ resource and observe $Y_{t,i}$
			\IF{$j \ne u$} 
				\IF{$Y_{t,a} = 1$ for any $a \in A_t$} 
					\STATE Set $u=j, j = u - \floor{(u - l)/2}, W_c=0$
					\STATE $\forall i \in A_t:$ set $S_i = S_i+Y_{t,i}, F_i = F_i+1-Y_{t,i} + Z_i$, $\forall k \in [K]\setminus A_t: F_k = F_k + Z_k$, $\forall i \in [K]: Z_i =0$
				\ELSE
					\STATE Set $W_c = W_c + 1$, and $\forall i \in A_t, Z_i = Z_i + 1$
					\STATE If $W_c = W_\delta$ then set $l=j, j = l + \ceil{(u - l)/2}$, $W_c=0, \forall i \in [K]: Z_i =0$
				\ENDIF
			\ELSE
				\STATE $\forall i \in A_t: S_i = S_i + Y_{t,i}, F_i = F_i + 1 - Y_{t,i}$
			\ENDIF
		\ENDFOR
	\end{algorithmic}
\end{algorithm}
\vspace{-2.25mm}
If reward $0$ is observed for all arms in the set $A_t$, $Z_i$ is incremented by $1$ for each arm $i \in A_t$ and variable $W_c$ is incremented by $1$ (line $13$). Variable $W_c$ counts the number of rounds for which reward is not observed on all the arms that are allocated resource.  If $W_c$ equals $W_\delta$, then with high probability $\hat{\theta}_i$ is possibly an underestimate of allocation equivalent. So all candidates smaller than the current value of $\hat{\theta}_s$ in set $\Theta$ are removed, and the search is repeated, starting with the middle element in the remaining half. $W_c$ as well as $Z_i,\; \forall i \in [K]$ are reset to $0$ (lines $14$). Resetting $Z_i$ values to zero once the number of zeros observed reaches $W_\delta$ ensures that they do not add to $F_i$ values when the resource is over-allocated. 

Since $\Theta$ has a finite size, the search for an allocation equivalent of  $\hat\theta$  terminates in the finite number of rounds with high probability. Once this happens, the algorithm allocates a resource to $C/\hat\theta_s$ arms (from Lemma \ref{lem:thetaSet}) in the subsequent rounds and observes their reward samples, i.e., a fixed number of arms are played (multiple-play) in each round. Also, the  $(C/\hat\theta_s)$ arms selected corresponds to top arms with the highest estimated means (line $7$), which are generated based on Thompson Sampling. Hence after finding allocation equivalent of $\theta_s$, our algorithm is the same as MP-TS. We leverage this observation to adapt the regret bounds of MP-TS.

\subsection{Analysis of ONUM-ST}
\label{ssec:analysisNUM_S}
When $\hat{\theta}_s$ is an overestimate, and no reward is observed for consecutive $W_\delta$ rounds, then $\hat{\theta}_s$ will be increased. Such increment leads to an incorrect estimate of $\hat{\theta}_s$. Hence, the value of $W_\delta$ is set such that the probability of having the wrong allocation equivalent is upper bounded by $\delta$. Let $T_{\theta_s}$ denote number of rounds needed to find an allocation equivalent of $\theta_s$ in $\Theta$. Our next result gives a high probability bound on it.

\begin{restatable}{lem}{sameThresholdEstRounds}
	\label{lem:sameThresholdEstRounds}
	Let $(\bmu, \theta_s, C)$ be an   instance such that $\mu_K \geq  \epsilon>0$. Then with probability at least $1-\delta$, the number of rounds needed by ONUM-ST to find the allocation equivalent of $\theta_s$ is upper bounded as 
	\begin{equation*}
	T_{\theta_s}\le \frac{\log(\log_2(K)/\delta)}{\log\left({1}/{(1-\epsilon)}\right)}\log_2(K).
	\end{equation*}
\end{restatable}
\noindent
This result extends Lemma 2 in \cite{NeurIPS19_verma2019censored}. The proof follows by binary search arguments and noting that one can come out of an under-allocation with high probability by observing the arms for a sufficiently large number of rounds. The detailed proof is given in APPENDIX.  Once the allocation equivalent of $\theta_s$ is found, the regret of ONUM-ST in the subsequent rounds, denoted by $\Regret_T^s$ is upper bounded as given in Theorem \ref{thm:MPRegret}.
\begin{thm}
	\label{thm:MPRegret}
	Let $(\bmu, \theta_s, C) \in \ONUM$ such that $\mu_{M} > \mu_{M+1}$. The expected regret of ONUM-ST in $T$ rounds after identifying allocation equivalent of $\theta_s$ is upper bounded as     
	\begin{equation*}
	\EE{\Regret_T^s} \le O\left((\log T)^{{2}/{3}}\right) + \mbox{$\sum_{i \in [K]\setminus [M]}$} \frac{(\mu_M-\mu_i)\log {T}}{d( \mu_i,\mu_M)}.
	\end{equation*}
\end{thm}
\noindent
As $\hat\theta_s$ is allocation equivalent to $\theta_s$, the instances $(\bmu,\theta_s, C)$ and $(\bmu,\hat{\theta}_s, C)$ is having the same mean reward. After knowing $\hat\theta_s$, the expected regret of ONUM-ST is the same as solving a MP-MAB instance. Therefore, we can directly use Theorem 1 of \cite{ICML15_komiyama2015optimal} to get the above regret bounds by setting $L=M$.

The assumption $\mu_{M} > \mu_{M+1}$ ensures that $KL$ divergence in the bound is well defined. It is also equivalent to assume that the set of top-$M$ arms is unique.
For a instance $(\bmu,\btheta, C)\in \ONUM$ and any feasible allocation $\bx\in \A$, we define the sub-optimality gap as $\Delta_x = \sum_{i=1}^K\mu_i\big(\one{x_i^\star \ge \theta_i} - \one{x_i \ge \theta_i}\big)$. The maximum regret incurred in a round is $\Delta_m = \max\limits_{\bx \in \A } \Delta_x$.
\begin{thm}
	\label{thm:regretHighConf}
	Let $(\bmu,\theta_s, C)\in \ONUM$, $\mu_K\geq \epsilon>0$, $\mu_{M} > \mu_{M+1}$, $W_\delta = {\log(\log_2(K)/\delta)}/{\log(1/(1-\epsilon))}$, and $T>W_\delta\log_2{(K)}$. Then with probability at least $1-\delta$, the expected regret of ONUM-ST is upper bounded as
	\begin{align*}
		\EE{\Regret_T} &\le  W_\delta\log_2{(K)}\Delta_m + O\left((\log T)^{{2}/{3}}\right) \\
		&\qquad  + \mbox{$\sum_{i \in [K]\setminus [M]}$} \frac{(\mu_M-\mu_i)\log {T}}{d( \mu_i,\mu_M)}.
	\end{align*}
\end{thm}

\begin{proof}
	We divide the cumulative regret of ONUM-ST into the two parts: regret before finding a correct allocation equivalent $(\hat\theta_s)$ and regret after knowing allocation equivalent. The $\hat\theta_s$ estimation happens in $T_{\theta_s}$ rounds and returns a correct allocation equivalent with probability at least $1-\delta$. As $\Delta_m$ is the maximum regret that can be incurred in any round, the maximum regret incurred for estimating allocation equivalent is upper bounded by $\Delta_m T_{\theta_s}$. Given that $\hat\theta_s$ is correct, Theorem \ref{thm:MPRegret} gives the regret incurred after knowing $\hat\theta_s$. 
	Hence the expected regret of ONUM-ST is a sum of these two regret bounds, and it holds with probability at least $1-\delta$. 
\end{proof}

Note that the assumption $\mu_K\ge\epsilon>0$ is only required to guarantee that the allocation equivalent of the threshold is found in finite time. This assumption is not required to get the upper bound on expected regret after knowing the allocation equivalent.

\begin{restatable}{cor}{regretSameThreshold}
	\label{cor:regretSameThreshold}
	Let assumptions in Theorem \ref{thm:regretHighConf} hold and set $\delta=T^{-(\log T)^{-\alpha}}$ in ONUM-ST such that $\alpha>0$. Then the expected regret of ONUM-ST is upper bounded as 
	\begin{align*}
		\EE{\Regret_T} &\le  O\left((\log T)^{1-\alpha}\right) + O\left((\log T)^{{2}/{3}}\right)\\
		&\qquad + \mbox{$\sum_{i \in [K]\setminus [M]}$} \frac{(\mu_M-\mu_i)\log {T}}{d( \mu_i,\mu_M)}.
	\end{align*}
\end{restatable}
\begin{proof}
	The bound follows from Theorem \ref{thm:regretHighConf} by setting $\delta=T^{-(\log T)^{-\alpha}}$ where $W_\delta = O(\log T)^{1-\alpha})$ and unconditioning the expected regret obtained in Theorem \ref{thm:regretHighConf}.
\end{proof}

\begin{cor}
	\label{cor:OptimalBoundST}
	The ONUM-ST is asymptotically optimal.
\end{cor}
\noindent
The first term in the regret bound of Corollary \ref{cor:regretSameThreshold} corresponds to the number of rounds needed to find an allocation equivalent, and the rest of it corresponds to the expected regret after knowing allocation equivalent. The proof of Corollary \ref{cor:OptimalBoundST} follows by comparing above bound with the lower bound in Eq. 1. %

	\section{Different Thresholds for All Users}
	\label{ssec:different_theta}

Now we consider a general case where the threshold may not be the same for all arms. The first difficulty with this setup is to estimate the threshold for each of the arms. Unfortunately, we do not have a result equivalent of Lemma \ref{lem:thetaSet} so that the search space can be restricted to a finite set. We need to search over the entire $[0, C]$ interval for each arm. The second difficulty is to find an optimal allocation which need not be just allocating resource to top $M$ arms. To see this, consider a  problem instance $(\bmu, \btheta, C)$ with $\bmu = (0.9,0.6,0.4)$, $\btheta = (0.6, 0.55, 0.45)$, and $C=1$. The optimal allocation is $\bx^\star = (0, 0.55, 0.45)$ with no resource allocated to the top arm. Our first result gives the optimal allocation for an instance with different thresholds in $\ONUM$. Let $KP(\bmu,\btheta, C)$ denote a $0$-$1$ knapsack problem with capacity $C$ and $K$ items where item $i$ has weight $\theta_i$ and value $\mu_i$. 

\begin{restatable}{prop}{diffThetaOptiSoln}
    \label{prop:diffThetaOptiSoln}
    Let $P=(\bmu,\btheta,C) \in \ONUM$. Then the optimal allocation for $P$ is a solution of $KP(\bmu,\btheta, C)$ problem.
\end{restatable}
\noindent
Assigning $\theta_i$ resource to arm $i$ increases the total mean reward by an amount $\mu_i$. As the goal is to allocate a resource such that the total mean reward is maximized, i.e., $\max\limits_{\bx \in \A}$  $\sum_{i\in[K]}\mu_i\one{x_i \ge \theta_i}$. It is equivalent to solving a 0-1 knapsack with capacity $C$ where item $i$ has weight $\theta_i$ and value $\mu_i$.

Let $l = C - \sum_{i: x_i^\star \ge \theta_i}\theta_i$ for an instance $P:=(\bmu,\btheta, C)$, where $r$ is the leftover resource after doing optimal allocation and recall that $\bx^\star=(x_1^\star, \ldots, x_K^\star)$ is the optimal allocation. Define $\gamma:=l/K$. Note that any problem instance having $\gamma= 0$ becomes a `hopeless' problem because the only threshold vector that is allocation equivalent to $\btheta$ is $\btheta$ itself, i.e., $x_i^\star=\theta_i, \;\forall i \in [K]$, which needs $\theta_i$ values to be estimated with full accuracy to obtain the optimal allocation. But if $\gamma>0$, then optimal allocation can be found with a small error in the estimates of $\theta_i$, as shown in the next result.
\vspace{-1mm}
\begin{restatable}{lem}{diffTheteEst}
	\label{lem:diffTheteEst}
	Let $\gamma>0$, $C \ge  \gamma + \min_{i \in [K]} \theta_i$,  and $\forall i \in [K], \hat\theta_i \in [\theta_i, \theta_i + \gamma]$. Then for any $\bmu \in [0,1]^K$, the instances $(\bmu, \btheta, C)$ and $(\bmu, \hat{\btheta}, C)$ are allocation equivalent.
\end{restatable}
\vspace{-1mm}
\noindent
Let $L^\star = \left\{i: x_i^\star \ge \theta_i\right\}$ and $l =C - \sum_{i: x_i^\star \ge \theta_i}\theta_i$. Since $l < \min_{i \in  K\setminus L^\star} \theta_i$, no reward can be obtained from any arm $i \in [K]\setminus L^\star$. If the leftover resource $l$ is uniformly distributed among all the arms i.e.,  increasing resource of each by an amount $\gamma=l/K$ for each arm, the optimal total mean reward still remains same. If threshold estimate of each arm $i \in [K]$ lies in $[\theta_i, \theta_i + \gamma]$, then by using Theorem $3.2$ of \cite{DO13_hifi2013sensitivity}, $KP(\bmu, \btheta, C)$ and $KP(\bmu, \hat\btheta, C)$ have the same optimal solution because of the total mean reward observed for instance $(\bmu, \btheta, C)$ and instance $(\bmu, \hat\btheta, C)$ is same.

Once the allocation equivalent of $\btheta$ is known, the problem is equivalent to solving a $KP(\bmu,\hat{\btheta}, N)$ which is equivalent to solving a Combinatorial Semi-Bandits \cite{ICML18_wang2018thompson} as shown in \cite{NeurIPS19_verma2019censored}. Combinatorial Semi-Bandits is the generalization of Multiple-Play Multi-Armed Bandits, where the size of superarms need not be identical in each round.

We develop an algorithm named Online Network Utility Maximization with the Different Threshold (ONUM-DT). It exploits result of Lemma \ref{lem:diffTheteEst} to find allocation equivalent and minimizes the regret using an algorithm from Combinatorial Semi-Bandits. The pseudo-code of ONUM-DT is given in Algorithm 2. ONUM-DT works as follows: it takes $K, C, \delta, \epsilon$ and $\gamma$ as input. We initialize the prior distribution of each arm as the Beta distribution $\beta(1, 1)$ which is same as in ONUM-ST. For every arm $i \in  [K]$, a binary search is performed over the interval $[0, C]$ and the variables $\hat{\theta}_{i}$, $\hat{\theta}_{t,i}$, $\hat{\theta}_{l,i}, \hat\theta_{u,i}, \hat\theta_{g,i}$ are tracked where $\hat{\theta}_{i}$ is the estimated value of ${\theta}_{i}$, $\hat{\theta}_{t,i}$ is the current estimate of $\theta_i$ and initialized by $C/K$; $\hat{\theta}_{u,i}$ and $\hat{\theta}_{l,i}$ denote the upper and lower bound of the binary search region for arm $i$; and $\hat{\theta}_{g,i}$ indicates whether current estimate lies in the interval $[\theta_i, \theta_i + \gamma]$ (line $2$). $Z_i$ keeps count of consecutive $0$ on arm $i$ when it is allocated resource. $Z_i$ changes to $0$ either after observing a reward or if no reward is observed for consecutively $W_\delta$ rounds. Let $S_i(t)$ and $F_i(t)$ denote the value of $S_i$ and $F_i$ at the start of round $t$. In round $t$, for each $i \in [K]$ an independent sample of $\hat\mu_{t,i}$ is drawn from $\beta(S_i(t), F_i(t))$ (line $5$). 

ONUM-DT finds allocation equivalent of $\btheta$ by doing binary search for all $i$. We say that threshold estimate of arm $i$ is good, i.e., $\hat{\theta}_i \in [\theta_i, \theta_i + \gamma]$ is checked by condition $\hat\theta_{u,i}  - \hat\theta_{l,i} \le \gamma$. If the condition satisfies, then the estimated threshold of the arm is within the desired tolerance, and it is indicated by setting $\hat\theta_{g,i}=1$. Otherwise it is set to $0$. If threshold estimate of arm $i$ is good, we set $\hat\theta_i = \hat\theta_{u,i}$ (line $12$). $\hat\theta_i$ represents the threshold estimate of arm $i$ that is used after having $\hat\theta_{g,i}=1, ~\forall i \in [K]$. 
\begin{algorithm}[H] 
	\caption{\bf ONUM-DT}
	\label{alg:ONUM-DT}
	\begin{algorithmic}[1]
		\STATE \textbf{Input:} $K, C, \delta, \epsilon, \gamma$
		\STATE Initialize: $\forall i \in [K]: \hat\theta_{i} = C, \hat\theta_{1, i} = C/K , \hat\theta_{l,i} = 0, \hat\theta_{u,i} = C, \theta_{g,i} = 0, S_i = 1, F_i = 1, Z_i =0$ 
		\STATE Set $W_\delta = \log (K\log_2(\lceil 1 + C/\gamma\rceil)/\delta)/\log(1/(1-\epsilon))$
		\FOR{$t=1,2, \ldots,$}
			\STATE $\forall i \in [K]: \hat{\mu}_{t,i} \leftarrow \text{Beta}(S_i, F_i)$
			\IF{$\theta_{g,j} = 0$ for any $j \in [K]$}
				\STATE $\forall i \in [K]$, update $\hat\theta_{t,i}$ using Eq. \eqref{equ:updateTheta}. Allocate $\hat\theta_{t,i}$ resource to arm $i$ and observe $Y_{t,i}$
				\FOR{$i = \{1,2,\ldots, K\}$}
					\IF{$\theta_{g,i} = 0$ and $\hat\theta_{t,i}>0$}
						\STATE If $Y_{t,i}=1$ then set $\hat\theta_{u,i} = \hat\theta_{t,i}, S_i = S_i + 1, F_i = F_i + Z_i, Z_i =0$ otherwise  $Z_i = Z_i + 1$
						\STATE If {$Z_i= W_\delta$} then set $\hat\theta_{l,i} = \hat\theta_{t,i}, Z_i=0 $
						\STATE If $\hat\theta_{u,i} - \hat\theta_{l,i} \le \gamma$ then set $\theta_{g,i}=1$ and $\hat\theta_i = \hat\theta_{u,i}$
					\ELSIF{$\theta_{g,i} = 1$ and $\hat\theta_{t,i}=\hat\theta_i$}
						\STATE Set $S_i = S_i+Y_{t,i}$ and $F_i = F_i+1-Y_{t,i}$ 
					\ENDIF	
				\ENDFOR
			\ELSE
				\STATE $A_t \leftarrow$ Oracle$\big( KP(\hat\bmu_{t}, \hat\btheta, C)\big)$ 
				\STATE $\forall i \in A_t,$ allocate $\hat\theta_{i}$ resource and observe $Y_{t,i}$. Update $S_i = S_i+Y_{t,i}$ and $F_i = F_i+1-Y_{t,i}$
			\ENDIF
		\ENDFOR
	\end{algorithmic}
\end{algorithm}

If $\theta_{g,i}=0$ (line $6$) for some $i$,  $\hat\theta_{t,i}$ is updated (line $7$) after computing the following events:
\begin{align*}
	&B_i(t) = \left\{ \hat\theta_{t,i} \le C - \sum_{\substack{j \in [K]:\theta_{g,j}=0 \\ \hat\mu_{t,j}/\hat\theta_{t,j} > \hat\mu_{t,i}/\hat\theta_{t,i} }} \hat\theta_{t,j}  \right\}, \\
	&G_i(t) = \left\{ \hat\theta_{t,i} \le C - \sum_{\substack{j \in [K]:\theta_{g,j}=0}}\hspace{-1mm} \hat\theta_{t,j} -  \sum_{\substack{k \in [K],\theta_{g,k}=1\\ \hat\mu_{t,k}/\hat\theta_{t,k} > \hat\mu_{t,i}/\hat\theta_{t,i} }} \hspace{-2mm}\hat\theta_{t,k} \right\},\\
	&\mbox{and } E_\theta = \{\forall i \in [K]: \theta_{g,i} = 1\}.
\end{align*}

Event $E_\theta $ states that each arm has a good threshold estimate, which means ONUM-DT found the allocation equivalent for $\btheta$. In round $t$, event $B_i(t)$ is defined for arm $i$ with $\theta_{g,i}=0$ and indicates whether it can get resource or not. Event $G_i(t)$ is defined for arm with $\theta_{g,i}=1$ and indicates if it can get resource. By construction, event $B_i(t)$ does not happen for arms having a good threshold estimate, and event $G_i(t)$ does not happen for arms having a bad threshold estimate. The arms having the highest empirical reward to resource ratio, i.e., $\hat\mu_j/\hat\theta_{t,i}$ gets resource first followed by second highest. The resource is first allocated among arms having a bad threshold estimate to find allocation equivalent as soon as possible. The leftover resource is allocated to arms with a good threshold estimate to increase the reward. In round $t$, the $\hat\theta_{t,i}$ for arm $i$ is updated as follows:
\begin{align}
	\label{equ:updateTheta}
	\hat\theta_{t,i} = \begin{cases}
	\hat\theta_{u,i}  										 &\mbox{if $E_\theta$ or $G_i(t)$  happens} \\
	\frac{\hat\theta_{l,i} + \hat\theta_{u,i}}{2}  &\mbox{if $B_i(t)$ happens} \\
	0 															   & \mbox{Otherwise}
	\end{cases}.
\end{align}

If $\theta_{g,i}=0$ for any arm, $\hat\theta_{t,i}$ resource is allocated to each arm $i \in [K]$ and reward $Y_{t,i}$ is observed (line $7$). If reward $1$ is observed for arm $i$ with $\hat\theta_{g,i}=0$, then the upper bound of threshold is $\hat\theta_{t,i}$, i.e, $\theta_{u,i}=\hat{\theta}_{t,i}$ (line $10$). The success and failure counts are also updated as $S_i = S_i + 1, F_i = F_i + Z_i$, and $Z_i$ is reset to $0$. If reward $0$ is observed after allocating positive resource, $Z_i$ is incremented by $1$. If $0$ reward is observed for successive $W_\delta$ rounds for arm $i$ that have bad threshold estimate then it means that $\hat\theta_{t,i}$ is an underestimate of $\theta_i$. So, lower bound of threshold to $\hat\theta_{t,i}$, i.e, $\theta_{l,i}=\hat{\theta}_{t,i}$ and $Z_i$ is reset to $0$ (line $11$). For any arm $i$ having good threshold estimate and $\hat\theta_i = \hat\theta_{u,i}$, its success and failure counts are updated as $S_i = S_i + Y_{t,i}, F_i = F_i + 1 - Y_{t,i}$ (line $14$).

Once we have good threshold estimate for all arms, we could adapt to an algorithm that works well for Combinatorial Semi-Bandits, like SDCB \cite{NIPS16_chen2016combinatorial} and CTS \cite{ICML18_wang2018thompson}. SDCB uses the UCB type index, whereas CTS uses Thompson Sampling. We adapt the CTS to our setting due to its better empirical performance. Oracle uses $KL(\hat\bmu_t, \hat\btheta, C)$ to identify the arms in the round $t$ where the learner has to allocate resource (denoted as set $A_t$ in line $18$). $\hat\theta_i$ resource is allocated to each arm $i \in A_t$ and reward $Y_{t,i}$ is observed. Then $S_i = S_i + Y_{t,i}, F_i = F_i + 1 - Y_{t,i}$ are updated (line $19$).

\subsection{Analysis of ONUM-DT}
The value of $W_\delta$ in ONUM-DT is set such that the probability of threshold estimate does not lie in $[\theta_i, \theta_i+\gamma]$ for all arms is upper bounded by $\delta$. Our next result gives an upper bound on the number of rounds required to obtain the allocation equivalent $\hat{\btheta}$ with high probability.
\begin{restatable}{lem}{MultiTheta}
    \label{lem:MultiTheta}
    Let $(\bmu, \btheta, C) \in \ONUM$ such that $\gamma>0$ and $\mu_K \geq  \epsilon>0$. Then with probability at least $1-\delta$, the number of rounds needed by \textnormal{ONUM-DT} to find an allocation equivalent of $\btheta$ is upper bounded as 
    \begin{equation*}
    	T_{\theta_d}  \le \frac{{K\log(K\log_2(\lceil 1 +{C}/{\gamma}\rceil)/\delta)}}{\log(1/(1-\epsilon))} {\log_2 (\lceil 1 + {C}/{\gamma}\rceil)}.
    \end{equation*}
\end{restatable}

Let $\Delta_x$ and $\Delta_m$ be defined as in Section \ref{ssec:analysisNUM_S}. Let $\gamma>0$, $S_x=\{i:x_i \ge \theta_i\}$ for any feasible allocation $a$, $K_{max} = \max_{x \in \A}|S_x|$, and $k^\star = \min_{x^\star \in \A}|S_{x^\star}|$. Note that we redefine $W_\delta = \log (K\log_2(\lceil 1 + C/\gamma\rceil)/\delta)/\log(1/(1-\epsilon))$. We need the following results to prove the regret bounds.
\begin{thm}
	\label{thm:CTSRegret}
	Let $\hat\btheta$ be allocation equivalent to $\btheta$ for instance $(\bmu, \btheta, C)$. After knowing $\hat\btheta$, the expected regret of ONUM-DT in $T$ rounds is upper bounded by $\left(\sum_{i \in [K]} \max\limits_{S_x:i\in S_x}\frac{8|S_x|\log {T}}{\Delta_x - 2(k^\star{}^2 + 2)\eta}\right) + \left(\frac{KK_{max}^2}{\eta^2} + 3K\right)\Delta_m +\alpha_1\left(\frac{8\Delta_m}{\eta^2}\left(\frac{4}{\eta^2} + 1\right)^{k^\star} \log\frac{k^\star}{\eta^2}\right)$ for any $\eta$ such that $\forall \bx \in \A$, $\Delta_x>2(k^\star{}^2+2)\eta$ and $\alpha_1$ is a problem independent constant. 
\end{thm}
\noindent
Note that once the estimated $\hat{\btheta}$ is allocation equivalent to $\btheta$, the ONUM problem with the different thresholds is equivalent to solving a Combinatorial Semi-Bandits problem. The proof follows by verifying Assumptions $1-3$ of \cite{ICML18_wang2018thompson} for the Combinatorial Semi-Bandits setup and then applying their regret bounds. Assumption $1$ states that the mean reward of a superarm only depends on the mean rewards of its constituting arms, and distributions of the arms are independent (Assumptions $3$). Both these assumptions hold for our case. We next proceed to verify Assumption $2$. For a fixed allocation $\bx\in \A$, the mean reward collected from vector $\bmu$ is given by $r(S,\bmu)=\sum_{i \in S}\bmu_i$ where $S=\left\{i:x_i \ge \hat\theta_i\right\}$. For any two reward vectors $\bmu$ and $\bmu^\prime$, we have
\begin{align*}
	r&(S, \bmu)-r(S, \bmu^\prime)=\sum_{i \in S}(\mu_i - \mu_i^\prime) \\
	& = \sum_{i \in [K]}  \one{x_i \ge \hat\theta_i}\left (\mu_i -\mu_i^\prime \right) \hspace{1.5mm} \text{$\Bigg($as $\sum_{i \in S}\mu_i = \sum_{i \in [K]} \mu_i \one{x_i \ge \hat\theta_i}\Bigg)$}\\
	& \le \sum_{i \in [K]}  \left (\mu_i -\mu_i^\prime \right)\leq    \sum_{i \in [K]}   |\mu_i -\mu_i^\prime |
	= B\parallel \bmu- \bmu^\prime \parallel_1
\end{align*}
where $B=1$. After knowing the allocation equivalent, the allocation to each arm remains the same in every subsequent round ($\hat{\theta}_i$ resource is allocated to arm $i \in A_t$). By using Theorem $1$ of \cite{ICML18_wang2018thompson} with parameter $B=1$, we get the regret bounds of Theorem \ref{thm:CTSRegret}.

\begin{thm}
	\label{thm:regretDiffThresholdHighConf}
	Let $(\bmu, \btheta, C)\in \ONUM$ such that $\gamma >0$, $\mu_K\geq \epsilon>0$, and $T>T_{\theta_d}$. Then with probability at least $1-\delta$, the expected regret of ONUM-DT is upper bounded by ${\Delta_m KW_\delta\log_2 \left(\lceil 1+C/\gamma\rceil \right)}  + \left(\sum_{i \in [K]} \max\limits_{S_x:i\in S_x}\frac{8|S_x|\log {T}}{\Delta_x - 2(k^\star{}^2 + 2)\eta}\right) + \left(\frac{KK_{max}^2}{\eta^2} + 3K\right)\Delta_m +\alpha_1\left(\frac{8\Delta_m}{\eta^2}\left(\frac{4}{\eta^2} + 1\right)^{k^\star} \log\frac{k^\star}{\eta^2}\right).$
\end{thm}
\noindent
The first term of regret bound corresponds to the regret incurred for finding the correct allocation equivalent with high probability. The number of rounds needed to find the correct allocation equivalent is $T_{\theta_d}$. As $\Delta_m$ is the maximum regret that can be incurred in any round, the maximum regret incurred for estimating allocation equivalent is upper bounded by $\Delta_m T_{\theta_d}$. The other terms correspond to the regret incurred after knowing the allocation equivalent. Once an allocation equivalent is known, the expected regret incurred is upper bounded as given in Theorem \ref{thm:CTSRegret}.  Hence the expected regret of ONUM-DT is a sum of these two regret bounds, and it holds with probability at least $(1-\delta)$.

\begin{restatable}{cor}{regretDiffThreshold}
	\label{cor:regretDiffThreshold}
	 Assume technical conditions stated in Theorem \ref{thm:regretDiffThresholdHighConf} hold. Set $\delta=1/T$ in ONUM-DT. Then the expected regret of ONUM-DT is upper bounded by
	$
	\Delta_m KW_\delta\log_2 \left(\lceil 1+C/\gamma\rceil \right) 
	+ \left(\sum_{i \in [K]} \max\limits_{S_x:i\in S_x}\frac{8|S_x|\log {T}}{\Delta_x - 2(k^\star{}^2 + 2)\eta}\right) 
	+ \left(\frac{KK_{max}^2}{\eta^2} + 3K\right)\Delta_m
	+\alpha_1\left(\frac{8\Delta_m}{\eta^2}\left(\frac{4}{\eta^2} + 1\right)^{k^\star} \log\frac{k^\star}{\eta^2}\right)
	$ 
	where $W_\delta= {\log(KT\log_2(\lceil 1 +{C}/{\gamma}\rceil))}/{\log(1/(1-\epsilon))}$. 
\end{restatable}
\noindent
The above bound follows from Theorem \ref{thm:regretDiffThresholdHighConf} with $\delta=1/T$ and unconditioning the expected regret obtained in Theorem \ref{thm:CTSRegret}.

	\section{Experiments}
	\label{sec:experiments}

We evaluate the performance of ONUM-ST and ONUM-DT empirically on three synthetically generated instances. In instance $1$, the threshold is the same for all arms, whereas, in instances $2$ and $3$, thresholds vary across arms. We ran the algorithm for $T=10000$ rounds in all the simulations. All the experiments are repeated $50$ times, and the regret curves are shown with a $95\%$ confidence interval. The vertical line on each curve shows the confidence interval. The following empirical results validate sub-linear bounds for our algorithms. The details about the problem instances are as follows:
~\\
\textbf{Instance $1$ (Identical Threshold):} It has $K = 50, C=20,$ $\theta_s=0.7, \delta=0.1$ and $\epsilon=0.1$.  The mean reward of arm $i\in [K]$ is $0.25 + (i-1)/100$. 
~\\
\textbf{Instance $2$ (Different Thresholds):} It has $K = 5,C=2,$ $\delta=0.1,$ $\epsilon=0.1$ and $\gamma=10^{-3}$. The mean reward vector is $\bmu = [0.9,$ $0.89,0.87,0.6,0.3]$ and the corresponding  threshold vector is $\btheta=[0.7,0.7,0.7,0.6,0.35]$. 
~\\
\textbf{Instance $3$ (Different Thresholds):} It has $K=10, C=3,$ $\delta=0.1$, $\epsilon=0.1$ and $\gamma=10^{-3}$. The mean reward vector is $\bmu \hspace{-0.2mm}=\hspace{-0.2mm} [0.9, 0.8, 0.42, 0.6, 0.5, 0.2, 0.11, 0.7, 0.3, 0.98]$ and the corresponding threshold vector is $\btheta = [0.6, 0.55, 0.3,$ $ 0.46, 0.34, 0.2, 0.07, 0.3, 0.25, 0.8]$. 

We considered two different reward distributions of arms: 1) Bernoulli, where the rewards of arm $i$ are Bernoulli distributed with parameter $\mu_i$, and 2) Uniform, where the rewards of arm $i$ is uniformly distributed in the interval $[\mu_i - 0.1, \mu_i+0.1]$. For any continuous reward distribution with support in $(0,1]$, the value of $W_\delta$ is set to $1$ because the reward is observed with probability $1$ when the allocated resource is above its threshold on any arm. For the Bernoulli distribution, the value of $W_\delta$ is $38$ for instance $1$, $62$ for instance $2$ and $69$ for instance $3$. Hence, we observe less regret for uniformly distributed rewards than Bernoulli distributed rewards. This difference is more significant when the arms have different thresholds.

\textbf{Experiments with the same threshold:} We perform two different experiments on problem instance $1$ using ONUM-ST. First, we varied the amount of resource $C$ while keeping other parameters unchanged. With more resource, the learner can allocate resource to more arms. Hence learner can observe rewards from more arms in each round, which leads to faster learning and low cumulative regret, as shown in Fig. \eqref{fig:UniSameThetaC} for the uniformly distributed rewards. For the uniform distribution we use binarization trick \cite{COLT12_agrawal2012analysis} to apply ONUM-ST: when a real-valued reward $Y_{t,i}\in (0,1]$ is observed, the algorithm is updated with a fake binary reward that is drawn from Bernoulli distribution with parameter $Y_{t,i}$, i.e., $Y_{t,i}^f \sim Ber(Y_{t,i}) \in \{0,1\}$. The different amount of resource has different optimal allocation and sub-optimality gap. Hence with large $W_\delta$ value for Bernoulli distributed rewards, we may not observe similar behavior (less regret with more resource)  as shown in Fig. \eqref{fig:BerSameThetaC}. 
\begin{figure}[!ht]
	\centering
	\begin{subfigure}[b]{0.24\textwidth}
		\includegraphics[width=\linewidth]{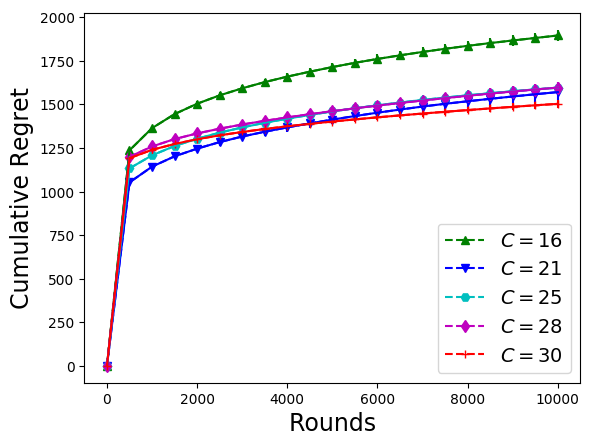}
		\caption{\footnotesize Bernoulli Distributed Reward}
		\label{fig:BerSameThetaC}
	\end{subfigure}
	\begin{subfigure}[b]{0.24\textwidth}
		\includegraphics[width=\linewidth]{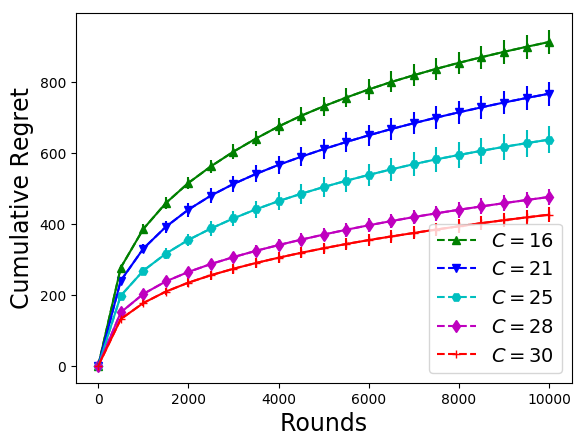}
		\caption{\footnotesize Uniform Distributed Reward}
		\label{fig:UniSameThetaC}	
	\end{subfigure}
	\caption{Regret of ONUM-ST.}
	\label{fig:SameThetaC}
\end{figure}

Second, we varied the threshold $\theta_s$ while keeping other parameters unchanged. As a smaller threshold allows the allocation of resource to more arms, we observe that a smaller threshold leads to faster learning due to more feedback. These trends are shown in Fig. \eqref{fig:BerSameThetaT} and \eqref{fig:UniSameThetaT} for Bernoulli and uniformly distributed rewards, respectively.
\begin{figure}[!ht]
	\begin{subfigure}[b]{0.24\textwidth}
		\includegraphics[width=\linewidth]{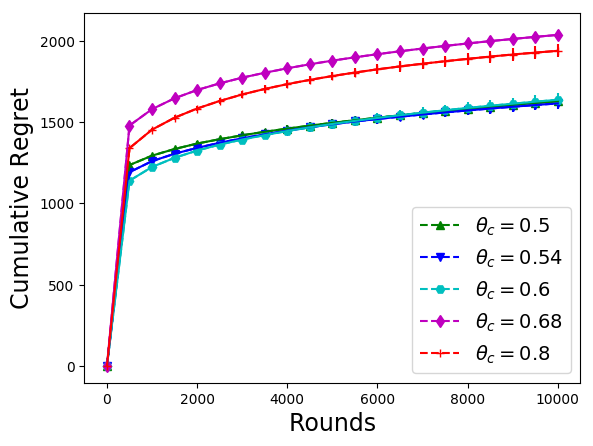}
		\caption{\footnotesize Bernoulli Distributed Reward}
		\label{fig:BerSameThetaT}
	\end{subfigure}
	\begin{subfigure}[b]{0.24\textwidth}
		\includegraphics[width=\linewidth]{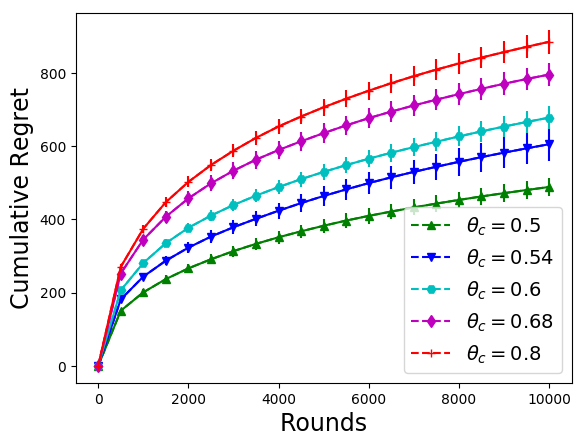}
		\caption{\footnotesize Uniform Distributed Reward}
		\label{fig:UniSameThetaT}
	\end{subfigure}
	\caption{Regret of ONUM-ST.}
	\label{fig:SameThetaQ}
\end{figure}

\textbf{Experiments with different thresholds:} 
We evaluate the performance of ONUM-DT on problem instances $2$ and $3$. We varied the amount of resource $C$ while keeping other parameters unchanged. As the thresholds are different across arms, an increase in the resource may lead to a selection of a different set of arms leading to different sub-optimality gaps. Hence, it does not show the same behavior (less regret with more resource) as observed for the same threshold. 
\begin{figure}[!ht]
	\centering
	\begin{subfigure}[b]{0.24\textwidth}
		\includegraphics[width=\linewidth]{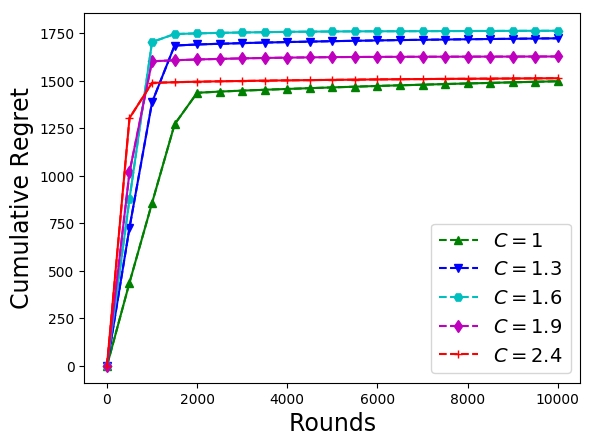}
		\caption{\footnotesize Bernoulli Distributed Reward}
		\label{fig:BerDiffTheta2}
	\end{subfigure}
	\begin{subfigure}[b]{0.24\textwidth}
		\includegraphics[width=\linewidth]{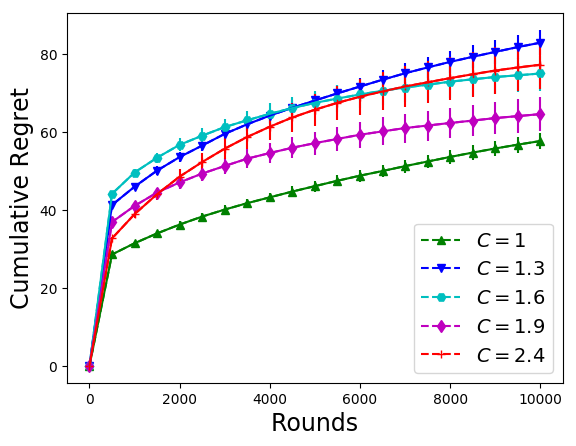}
		\caption{\footnotesize Uniform Distributed Reward}
		\label{fig:UniDiffTheta2}
	\end{subfigure}
	\caption{Regret of ONUM-DT.}
	\label{fig:DiffTheta1}
\end{figure}
\noindent
But we observe that the allocation equivalent is learned faster as the reward of more arms can be observed simultaneously with more resource. These observations are shown in Figs. \eqref{fig:BerDiffTheta2} and \eqref{fig:UniDiffTheta2} generated on instance $2$ for Bernoulli and uniformly distributed rewards on instance $2$, and same is repeated in Figs. \eqref{fig:BerDiffTheta3} and \eqref{fig:UniDiffTheta3} on instance $3$. We run experiment $200$ times for uniformly distributed rewards (Figs. \eqref{fig:UniDiffTheta2} and \eqref{fig:UniDiffTheta3}) on instance $2$ and $3$ as confidence intervals overlapped for $50$ runs.

\begin{figure}[!ht]
	\begin{subfigure}[b]{0.24\textwidth}
		\includegraphics[width=\linewidth]{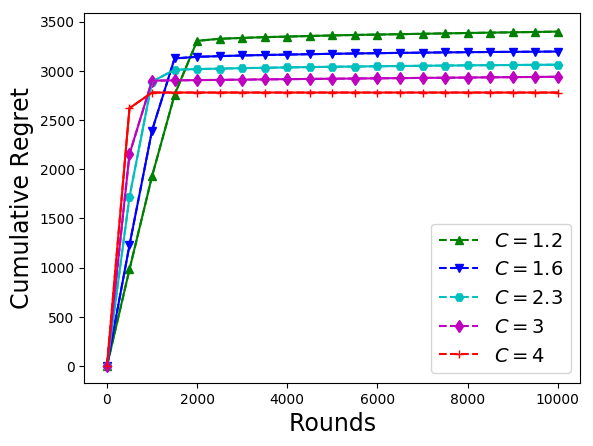}
		\caption{\footnotesize Bernoulli Distributed Reward}
		\label{fig:BerDiffTheta3}
	\end{subfigure}	
	\begin{subfigure}[b]{0.24\textwidth}
		\includegraphics[width=\linewidth]{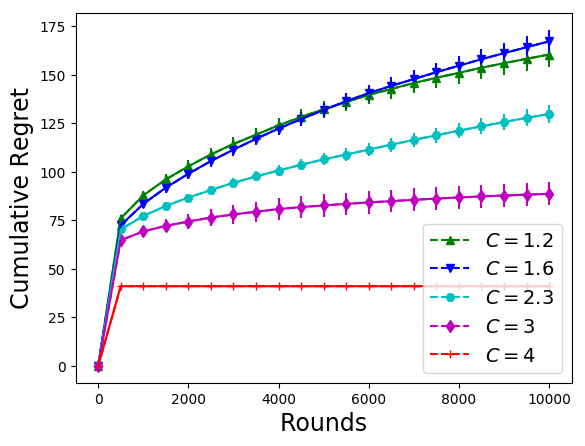}
		\caption{\footnotesize Uniform Distributed Reward}
		\label{fig:UniDiffTheta3}
	\end{subfigure}
	\caption{Regret of ONUM-DT.}
	\label{fig:DiffTheta2}
\end{figure}

	\section{Conclusion and Future Extensions}
	\label{sec:conclusion}

We proposed a novel framework for Online Network Utility Maximization (ONUM) with unknown utilities. We focused on threshold type utilities where each agent gets non-zero utility only when its allocated resource is higher than some threshold. The goal is to assign resource among agents such that the total expected utility is maximized. We considered two variants of the problem depending on whether thresholds are identical across the arms (symmetric) or not (asymmetric). Using the concept of `allocation equivalent,' and its connection to Multiple-Play Multi-Armed Bandits, we developed an optimal algorithm named ONUM-ST for the symmetric case. For the asymmetric case, we established that it is connected to a more general Combinatorial Semi-Bandits setup and developed an algorithm named ONUM-DT. Both algorithms achieve logarithm regret.

In our work, we assumed that a lower bound of the mean utilities is known, and it is also required knowledge of horizon $T$ to achieve logarithms regret. It would be interesting to see if logarithm regret can be achieved without such assumptions.

	\appendix
	\label{sec:appendix}

	\noindent
	\textbf{Proof of Lemma \ref{lem:sameThresholdEstRounds}.}
	The proof is adapted from Lemma $2$ of \cite{NeurIPS19_verma2019censored} by allowing $\theta_s \in [0,C]$. Note that when $\hat{\theta}_s > \theta_s$, it can happen that no reward is observed for consecutive $W_\delta$ rounds and leads to incorrect estimation of $\theta_s$. We want to set $W_\delta$ such a way that the probability of occurring of such event is upper bounded by $\delta$.
	
	Let $E_{\hat\theta_s}$ be the event that no reward is observed on $C/\hat{\theta}_s$ arms for $W_\delta$ consecutive rounds when $\hat{\theta}_s>\theta_s$. As $(1 - \mu_i)$ is the probability of not observing reward at arm $i$, the probability of the event $E_{\hat\theta_s}$ is bounded as follows:
	\begin{align*}
		&\Prob{\mbox{$E_{\hat\theta_s}$ occurs $| \hat\theta_s$ $1$st used at $T_{\hat\theta_s}$}} = \prod_{w=T_{\hat\theta_s}}^{T_{\hat\theta_s} + W_\delta-1} \hspace{-2mm} \prod_{i \in A_{w}} (1 - \mu_i)
	\end{align*}
	As rewards are i.i.d., $\mu_K \geq  \epsilon>0$ and $\hat\theta_s \in [0,C]$, we have
	\vspace{-1mm}
	\begin{align*}
		&\qquad \le  \prod_{w=T_{\hat\theta_s}+1}^{T_{\hat\theta_s} + W_\delta} (1 - \epsilon)^{\frac{C}{\hat\theta_s}} = (1 - \epsilon)^{\frac{CW_\delta}{\hat\theta_s}} \le (1 - \epsilon)^{W_\delta}
	\end{align*}
	Since we are doing binary search, the algorithm goes through at most $\log_2(K)$ overestimates of $\theta_s$.
	\vspace{-1mm}
	\begin{align*}
		&\Prob{\mbox{$E_{\hat\theta_s}$ for any overestimated $\hat\theta_s$}} \le (1 - \epsilon)^{W_\delta}\log_2(K)
	\end{align*}
	We bound the probability of making mistake by $\delta$ and get,
	\vspace{-1mm}
	\begin{equation*}
		(1 - \epsilon)^{W_\delta}\log_2(K) \le \delta \implies (1 - \epsilon)^{W_\delta} \le \delta/\log_2(K)
	\end{equation*}
	Taking log both side, we have
	\vspace{-1mm}
	\begin{align*}
		&W_\delta\log(1 - \epsilon) \le \log(\delta/\log_2(K))\\ 
		&\implies W_\delta \ge \frac{\log(\log_2(K)/\delta)}{\log\left({1}/{(1 - \epsilon)}\right) }
	\end{align*}
	$W_\delta$ is set as above
	so that ONUM-ST finds correct allocation equivalent with probability at least $1-\delta$ in $W_\delta\log_2(K)$ rounds. 

	\noindent
	\textbf{Proof of Lemma \ref{lem:MultiTheta}.}
	The proof is adapted from Lemma $4$ of \cite{NeurIPS19_verma2019censored} by allowing $\theta_i \in [0,C]$.
	For any arm $i \in [K]$, we want $\hat{\theta}_i \in [\theta_i, \theta_i + \gamma]$ so we divide interval $[0,C]$ into a discrete set $\Theta \doteq \left\{0, \gamma, 2\gamma, \ldots, C\right\}$ and note that $|\Theta| = \left\lceil1+ {C}/{\gamma}\right\rceil$.
	
	Let $E_{\hat\theta_i}$ be the event that no reward is observed for consecutive $W_\delta$ rounds when $\hat{\theta}_i$ is overestimated.	As $(1 - \mu_i)$ is the probability of not observing reward for arm $i$ and $\mu_K\ge \epsilon$, the probability of happening $E_{\hat\theta_i}$ is bounded by $\delta$ as follows:
	\vspace{-1mm}
	\begin{equation*}
		\Prob{E_{\hat\theta_i} \mbox{ happens}} = (1 - \mu_i)^{W_\delta} \le (1 - \epsilon)^{W_\delta}
	\end{equation*}
	Since we are doing binary search, the algorithm goes through at most $\log_2(|\Theta|)$ overestimates of $\theta_i$. 
	\vspace{-1mm}
	\begin{equation*}
		\Prob{E_{\hat\theta_i} \mbox{ happens for any overestimate}} \le (1 - \epsilon)^{W_\delta}\log_2(|\Theta|)
	\end{equation*}
	Next, we will bound the probability of making mistake for any of the arm. That is given by
	\vspace{-1mm}
	\begin{align*}
		&\Prob{\exists i  \in [K], E_{\hat\theta_i} \mbox{ happens for any overestimate}} \\
		&\qquad \le \sum_{i=1}^{K}\Prob{E_{\hat\theta_i} \mbox{ happens for any overestimate}}  \\
		&\qquad \le K(1 - \epsilon)^{W_\delta}\log_2(|\Theta|)
	\end{align*}
	We bound the probability of making mistake by $\delta$ and get,
	\vspace{-1mm}
	\begin{equation*}
	K (1 - \epsilon)^{W_\delta}\log_2(|\Theta|) \le \delta \implies (1 - \epsilon)^{W_\delta} \le \delta/K \log_2(|\Theta|)
	\end{equation*}
	Taking log both side, we have
	\vspace{-1mm}
	\begin{align*}
		&W_\delta\log(1 - \epsilon) \le \log(\delta/K \log_2(|\Theta|))\\
		&\implies W_\delta \ge \frac{\log(K \log_2(|\Theta|)/\delta)}{\log\left({1}/{(1 - \epsilon)}\right) }
	\end{align*}
	We set $W_\delta = {\log(K \log_2(|\Theta|)/\delta)}/{\log\left({1}/{(1 - \epsilon)}\right) }$.
	Therefore, the minimum rounds needed for each arm $i$ to correctly find $\hat{\theta}_i$ with probability at least $1-\delta/K$ is upper bounded by $W_\delta\log_2(|\Theta|)$. Using union bound, all $\hat\theta_i \in [\theta_i, \theta_i+\gamma]$ are correctly estimated with probability at least $1-\delta$ in $KW_\delta\log_2(|\Theta|)$ rounds where $|\Theta| = \left\lceil1+ {C}/{\gamma}\right\rceil$.

	\section*{Acknowledgments}
	Manjesh K. Hanawal would like to thank the support from SEED grant (16IRCCSG010) from IIT Bombay, INSPIRE faculty fellowships from DST and Early Career Research (ECR) Award from SERB, Government of India.

	\bibliographystyle{IEEEtran}
	\bibliography{ref}

\begin{thebibliography}{10}
\providecommand{\url}[1]{#1}
\csname url@samestyle\endcsname
\providecommand{\newblock}{\relax}
\providecommand{\bibinfo}[2]{#2}
\providecommand{\BIBentrySTDinterwordspacing}{\spaceskip=0pt\relax}
\providecommand{\BIBentryALTinterwordstretchfactor}{4}
\providecommand{\BIBentryALTinterwordspacing}{\spaceskip=\fontdimen2\font plus
\BIBentryALTinterwordstretchfactor\fontdimen3\font minus
  \fontdimen4\font\relax}
\providecommand{\BIBforeignlanguage}[2]{{%
\expandafter\ifx\csname l@#1\endcsname\relax
\typeout{** WARNING: IEEEtran.bst: No hyphenation pattern has been}%
\typeout{** loaded for the language `#1'. Using the pattern for}%
\typeout{** the default language instead.}%
\else
\language=\csname l@#1\endcsname
\fi
#2}}
\providecommand{\BIBdecl}{\relax}
\BIBdecl

\bibitem{ETT1997_ChargingAndRateControl}
F.~P. Kelly, ``Charging and rate control for elastic traffic,'' \emph{European
  Transactions on Telecommunications}, vol.~8, no.~1, pp. 33--37, 1997.

\bibitem{ETT2008_StochasticNUM_YiChiang}
Y.~Yi and M.~Chiang, ``Stochastic network utility maximisation—a tribute to
  kelly's paper published in this journal a decade ago,'' \emph{European
  Transactions on Telecommunications}, vol.~19, no.~4, pp. 421--442, 2008.

\bibitem{INFOCOM2010_DelayBasedNUM_Neely}
M.~J. Neely, ``Delay based network utility maximization,'' in \emph{IEEE
  INFOCOM}, 2010.

\bibitem{WiOpt2017_DRUM_EryilmazKoprulu}
A.~Eryilmaz and I.~Koprulu, ``Discounted-rate utility maximization (drum): A
  framework for delay-sensitive fair resource allocation,'' in \emph{IEEE
  WiOpt}, 2017.

\bibitem{WiOpt2018_NUMHetrogeneous_SinhaModiano}
A.~Sinha and E.~Modiano, ``Network utility maximization with heterogeneous
  traffic flows,'' in \emph{IEEE WiOpt}, 2018.

\bibitem{JSAC2006_TutorialOnNUM_PalomarChinag}
D.~Palomar and M.~Chiang, ``A tutorial on decomposition methods for network
  utility maximization,'' \emph{IEEE Journal on Selected Areas in
  Communications}, vol.~24, no.~8, pp. 1439--1451, 2006.

\bibitem{TAC2007_AlternateDistributed_PalomarChinag}
D.~Palomar and M.~Chiang, ``Alternative distributed algorithms for network
  utility maximization: Framework and applications,'' \emph{IEEE Transaction on
  Automatic Control}, vol.~52, no.~12, pp. 2254--2269, 2007.

\bibitem{PEVA2013_NetworkUM_LiNeely}
C.-P. Li and M.~J. Neely, ``Network utility maximization over partially
  observable markovian channels,'' \emph{Perform. Eval.}, vol.~70, pp.
  528--548, 2013.

\bibitem{ML2002_FiniteTimeAnlaysis_Auer}
P.~Auer, N.~Cesa-Bianchi, and P.~Fischer, ``Finite-time analysis of the
  multiarmed bandit problem,'' \emph{Machine Learning}, vol.~47, no. 2--3, pp.
  235 --256, 2002.

\bibitem{Book2012_RegretAnlaysis_Bubeck}
S.~Bubeck and N.~Cesa-Bianchi, \emph{Regret analysis of stochastic and
  nonstochastic multi-armed bandit problems}, 2012.

\bibitem{JSAC2011_DistributedAlgorithms_Anandakumar}
A.~Anandkumar, N.~Michael, A.~Tang, and A.~Swami, ``Distributed algorithms for
  learning and cognitive medium access with logarithmic regret,'' \emph{IEEE
  Journal on Selected Areas in Communications}, vol.~29, no.~4, pp. 731--745,
  2011.

\bibitem{TCNS2016_DistributedAlgorithms_Anandakumar}
N.~Nayyar, D.~Kalathil, and R.~Jain, ``Regret-optimal learning in decentralized
  multi-player multi-armed bandits,'' \emph{IEEE Transactions on Control of
  Network Systems}, vol.~5, no.~1, pp. 597--606, 2016.

\bibitem{ALT2018_MultiplayerBandits_BessonKaufmann}
L.~Besson and E.~Kaufmann, ``Multi-player bandits models revisited,'' in
  \emph{Algorithmic Learning Theory (ALT)}, 2018.

\bibitem{INFOCOM2019_DistributedLearning_TibrewalPatchalaHanawal}
H.~Tibrewal, S.~Patchala, M.~Hanawal, and S.~Darak, ``Distributed learning and
  optimal assignment in multiplayer heterogeneous networks,'' in \emph{IEEE
  INFOCOM}, 2019.

\bibitem{WiOpt2019_DistributedAlgorithms_VermaHanawalVaze}
A.~Verma, M.~Hanawal, and R.~Vaze, ``Distributed algorithms for efficient
  learning and coordination in ad hoc networks,'' in \emph{IEEE WiOpt}, 2019.

\bibitem{INFOCOM2019_CombinatorialSleeping_LiLiuJi}
F.~Li, J.~Liu, and B.~Ji, ``Combinatorial sleeping bandits with fairness
  constraints,'' in \emph{IEEE INFOCOM}, 2019.

\bibitem{JACM18_badanidiyuru2018bandits}
A.~Badanidiyuru, R.~Kleinberg, and A.~Slivkins, ``Bandits with knapsacks,''
  \emph{Journal of the ACM (JACM)}, vol.~65, no.~3, p.~13, 2018.

\bibitem{NIPS16_abernethy2016threshold}
J.~D. Abernethy, K.~Amin, and R.~Zhu, ``Threshold bandits, with and without
  censored feedback,'' in \emph{Advances In Neural Information Processing
  Systems}, 2016, pp. 4889--4897.

\bibitem{ICML18_jain2018firing}
L.~Jain and K.~Jamieson, ``Firing bandits: Optimizing crowdfunding,'' in
  \emph{International Conference on Machine Learning}, 2018, pp. 2211--2219.

\bibitem{UAI14_lattimore2014optimal}
T.~Lattimore, K.~Crammer, and C.~Szepesv{\'a}ri, ``Optimal resource allocation
  with semi-bandit feedback,'' in \emph{Proceedings of the Thirtieth Conference
  on Uncertainty in Artificial Intelligence}.\hskip 1em plus 0.5em minus
  0.4em\relax AUAI Press, 2014, pp. 477--486.

\bibitem{NIPS15_lattimore2015linear}
T.~Lattimore, K.~Crammer, and C.~Szepesv{\'a}ri, ``Linear multi-resource
  allocation with semi-bandit feedback,'' in \emph{Advances in Neural
  Information Processing Systems}, 2015, pp. 964--972.

\bibitem{ALT18_dagan18a}
Y.~Dagan and C.~Koby, ``A better resource allocation algorithm with semi-bandit
  feedback,'' in \emph{Proceedings of Algorithmic Learning Theory}, 2018, pp.
  268--320.

\bibitem{NeurIPS19_verma2019censored}
A.~Verma, M.~Hanawal, A.~Rajkumar, and R.~Sankaran, ``Censored semi-bandits: A
  framework for resource allocation with censored feedback,'' in \emph{Advances
  in Neural Information Processing Systems}, 2019, pp. 14\,499--14\,509.

\bibitem{NIPS15_combes2015combinatorial}
R.~Combes, M.~S. T.~M. Shahi, A.~Proutiere \emph{et~al.}, ``Combinatorial
  bandits revisited,'' in \emph{Advances in Neural Information Processing
  Systems}, 2015, pp. 2116--2124.

\bibitem{ICML15_komiyama2015optimal}
J.~Komiyama, J.~Honda, and H.~Nakagawa, ``Optimal regret analysis of thompson
  sampling in stochastic multi-armed bandit problem with multiple plays,'' in
  \emph{International Conference on Machine Learning}, 2015, pp. 1152--1161.

\bibitem{NIPS16_chen2016combinatorial}
W.~Chen, W.~Hu, F.~Li, J.~Li, Y.~Liu, and P.~Lu, ``Combinatorial multi-armed
  bandit with general reward functions,'' in \emph{Advances in Neural
  Information Processing Systems}, 2016, pp. 1659--1667.

\bibitem{ICML18_wang2018thompson}
S.~Wang and W.~Chen, ``Thompson sampling for combinatorial semi-bandits,'' in
  \emph{International Conference on Machine Learning}, 2018, pp. 5101--5109.

\bibitem{TAC1987_MultiPlayBandits_Anatharam}
V.~Anantharam, P.~Varaiya, and J.~Walrand, ``Asymptotically efficient
  allocation rules for the multiarmed bandit problem with multiple plays- part
  {I},'' \emph{IEEE Transactions on Automatic Control}, vol.~32, no.~11, pp.
  968--976, 1987.

\bibitem{DO13_hifi2013sensitivity}
M.~Hifi and H.~Mhalla, ``Sensitivity analysis to perturbations of the weight of
  a subset of items: The knapsack case study,'' \emph{Discrete Optimization},
  vol.~10, no.~4, pp. 320--330, 2013.

\bibitem{COLT12_agrawal2012analysis}
S.~Agrawal and N.~Goyal, ``Analysis of thompson sampling for the multi-armed
  bandit problem,'' in \emph{Conference on Learning Theory}, 2012, pp. 39--1.

\end{thebibliography}

\end{document}